\documentclass[sigconf,nonacm]{acmart}

\usepackage{subcaption}
\usepackage{multirow}

\usepackage{etoolbox}   
\newbool{anonymous}
\setbool{anonymous}{false}
\usepackage{enumitem}

\newcommand{\rocauc}{AUC}
\newcommand{\logloss}{\emph{logloss}}
\newcommand{\OptHPs}{\emph{A-priori HPs}}
\newcommand{\optHPs}{\emph{a-priori HPs}}
\newcommand{\ValHPs}{\emph{Validation HPs}}
\newcommand{\valHPs}{\emph{validation HPs}}
\newcommand{\eg}{\emph{e.g.}}
\newcommand{\ie}{\emph{i.e.}}
\newcommand{\prm}{\emph{probability metric}}
\newcommand{\prms}{\emph{probability metrics}}
\newcommand{\cm}{\emph{label metric}}

\newcommand{\cms}{\emph{label metrics}}
\newtheorem{thm}{Theorem}
\newtheorem{lem}[thm]{Lemma}

\AtBeginDocument{%
  \providecommand\BibTeX{{%
    \normalfont B\kern-0.5em{\scshape i\kern-0.25em b}\kern-0.8em\TeX}}}

\setcopyright{acmcopyright}
\copyrightyear{2018}
\acmYear{2018}
\acmDOI{XXXXXXX.XXXXXXX}

\acmConference[CIKM2022]{ACM International Conference
on Information and Knowledge Management}{Oct 17--22, 2022}{Atlanta, Georgia}
\acmPrice{XXX}
\acmISBN{978-1-4503-XXXX-X/18/06}

\acmSubmissionID{6144}

\begin{document}

\title{To SMOTE, or not to SMOTE?}

\ifbool{anonymous}{\author{anonymous}}
{
\author{Yotam Elor}
\email{yotame@amazon.com}
\affiliation{%
  \institution{Amazon}
  \state{New York}
  \country{USA}
}

\author{Hadar Averbuch-Elor}
\email{hadarelor@cornell.edu}
\affiliation{%
  \institution{Cornell University}
  \state{New York}
  \country{USA}
}

\renewcommand{\shortauthors}{Elor and Elor}
}

\begin{abstract}
Balancing the data before training a classifier is a popular technique to address the challenges of imbalanced binary classification in tabular data. Balancing is commonly achieved by duplication of minority samples or by generation of synthetic minority samples. While it is well known that balancing affects each classifier differently, most prior empirical studies did not include strong state-of-the-art (SOTA) classifiers as baselines. In this work, we are interested in understanding whether balancing is beneficial, particularly in the context of SOTA classifiers. Thus, we conduct extensive experiments considering three SOTA classifiers along the weaker learners used in previous investigations. Additionally, we carefully discern proper metrics, consistent and non-consistent algorithms and hyper-parameter selection methods and show that these have a significant impact on prediction quality and on the effectiveness of balancing. Our results support the known utility of balancing for weak classifiers. However, we find that balancing does not improve prediction performance for the strong ones. We further identify several other scenarios for which balancing is effective and observe that prior studies demonstrated the utility of balancing by focusing on these settings.
\end{abstract}



\maketitle

\section{Introduction}
\label{sec:intro}

In binary classification problems the classifier is tasked with classifying each sample to one of two classes. When the number of samples in the majority (bigger) class is considerably larger then the number of samples in the minority (smaller) class, the dataset is considered imbalanced. This skewness is believed to be challenging as both the classifier training process and the common metrics used to assess classification quality tend to be biased towards the majority class.

Numerous metrics have been proposed to address the challenges of imbalanced binary classification, \eg{}, Brier score, the area under the receiver operating characteristic curve (\rocauc{}), $F_\beta$ and AM; see the survey of \cite{tharwat2020classification}. They are all non-symmetric and associate a higher loss with classifying a minority sample wrong compared to a majority sample. These metrics can be separated into two groups based on the underlying predictions used to compute them: 1) \emph{Probability metrics} are calculated using the predicted class probabilities (\emph{e.g.}, \rocauc{} and Brier score); and 2) \emph{label metrics} are based on the predicted labels (\emph{e.g.}, $f_\beta$, balanced-accuracy and Jaccard similarity coefficient).

Machine learning classifiers typically optimize a symmetric objective function which associates the same loss with minority and majority samples. Thus, considering imbalanced classification problems, as previously noted in \cite{herschtal2004optimising,branco2016survey}, there seems to be a discrepancy between the classifier's symmetric optimization and the non-symmetric metric of interest which might result in an undesirable bias in the final trained model.

On the other hand, theoretical investigations into proper metrics and consistent algorithms suggest that the aforementioned discrepancy might not hinder prediction performance after all; see Section \ref{sec:proper_consistent}. In summary, these theoretical results state that when considering proper metrics or consistent classifiers, optimizing a symmetric objective function (commonly \logloss{}) results in optimizing the non-symmetric metric of interest, suggesting that there is no need to address the data imbalance. Furthermore, some studies propose that skewed class distribution is not the only (or even the main) factor that hinders prediction quality for imbalanced datasets but the difficulty arises from other phenomenons that are common in imbalanced data. These other contributing factors could be small sample size, class separability issues, noisy data, data drift and more~\cite{lopez2013insight,stefanowski2016dealing,sun2009classification}.

In light of these theoretical results and ideas, we revisit the empirical evaluation of the techniques proposed to address the challenges of imbalanced data. These techniques could be roughly divided to three: 1) \emph{preprocessing (or balancing)}, that augment the data, usually to be more balanced, before training the classifier; 2) \emph{specialized classifiers}, that are specific for imbalanced data \cite{ling2004decision,rankboost,galar2011review}. In many cases, this is achieved by increasing the weight of the minority samples in the objective function \cite{domingos1999metacost}; and 3) \emph{postprocessing}, where the classifier predictions are augmented during inference \cite{nunez2011post,Calibration9137689}. Due to its vast popularity, we focus on reevaluating the utility of balancing schemes.

The simplest balancing methods are either oversampling the minority class by duplicating minority samples or under-sampling the majority class by randomly removing majority samples. The idea of adding synthetic minority samples to tabular data was first proposed in SMOTE~\cite{SMOTE}, where synthetic minority samples are created by interpolating pairs of the original minority points. SMOTE is extremely popular: $156$ \href{https://scholar.google.com}{Google Scholar} papers and $60$ \href{https://dblp.org}{dblp} papers with the word "SMOTE" in the title were published in $2021$, $87$ \href{https://stackoverflow.com/}{Stackoverflow} questions discussing SMOTE were asked and dozens of blog posts demonstrating how to use it were published in $2021$. Additionally, more than $100$ SMOTE extensions and variants were proposed. See \cite{SMOTE_SURVEY} for a comprehensive survey of SMOTE-like methods and \cite{StopOversampling9761871} for additional statistics demonstrating the popularity of balancing in the community.

To evaluate the utility of balancing and SMOTE-like schemes we perform extensive experiments that, among other technical details, discern proper metrics, consistent and non-consistent classifiers; consider several hyper-parameters (HPs) tuning schemes; and include strong SOTA classifiers. While it is well known that balancing affects each classifier differently (see \eg{}, \cite{branco2016survey,LOPEZ20126585}), to the best of our knowledge, this is the first large scale evaluation of balancing that includes strong modern classifiers as baselines. Moreover, including these strong classifiers is important because they yield the best prediction quality of all methods tested.

Our main results and contributions are:\footnote{The code used in the experiments and our full results can be found in
\ifbool{anonymous}
{\href{https://github.com/anonymous}{https://github.com/anonymous}}
{\href{https://github.com/aws/to-smote-or-not}{https://github.com/aws/to-smote-or-not}}
}
\begin{itemize}[leftmargin=*]
    \item To the best of our knowledge, this is the first large scale study of balancing that includes SOTA classifiers as baselines. Specifically, we use lightGBM\cite{lgbm}, XGBoost\cite{XGBoost} and Catboost\cite{prokhorenkova2018catboost}.
    \item When the objective metric is proper we empirically show that:
    \begin{itemize}[leftmargin=*]
        \item Balancing could improve prediction performance for weak classifiers but not for the SOTA classifiers.
        \item The strong classifiers (without balancing) yield better prediction quality than the weak classifiers with balancing.
    \end{itemize}
    \item When the objective is a \cm{}, we empirically show that for strong classifiers:
    \begin{itemize}[leftmargin=*]
        \item Balancing and optimizing the decision threshold provide similar prediction quality. However, optimizing the decision threshold is recommended do to simplicity and lower compute cost.
        \item When balancing (instead of optimizing the decision threshold) simple random oversampling is as effective as the advanced SMOTE-like methods.
    \end{itemize}
    \item We identify several scenarios for which SMOTE-like oversampling can improve prediction performance and should be applied.
    \item We tie our empirical results with theoretical grounding and show that under common assumptions \rocauc{} is proper. While this is a straightforward corollary of previous work, to the best of our knowledge, we are the first to explicitly state that.
\end{itemize}

\section{Proper metrics and consistent algorithms}
\label{sec:proper_consistent}
In this section, we discuss proper metrics and algorithm consistency which have a large impact on performance and specifically on the effectiveness of balancing. We then show that under common assumptions the AUC metric is proper.

Modern machine learning classifiers typically aim at predicting the underlying probabilities by optimizing the symmetric logistic loss (\logloss). For example, the popular boosting forest methods XGBoost~\cite{XGBoost}, Catboost~\cite{prokhorenkova2018catboost} and LightGBM~\cite{lgbm} optimize \logloss{} by default. Neural network methods commonly optimize the softmax loss~\cite{he2016deep,krizhevsky2012imagenet} which is equivalent to \logloss{} when the number of classes is two. Thus, when using a \cm{} (which is based on class predictions), there is a need to convert the predicted class probabilities to class predictions. This is usually done by using a decision threshold, i.e., a sample is predicted \emph{positive} if the predicted probability to be \emph{positive} is higher than a predetermined threshold and \emph{negative} otherwise. As we will discuss below, the decision threshold can significantly affect performance and choosing it correctly is important.

A \prm{} is proper when it is optimized by a classifier predicting the true class probabilities. For example, it is easy to see that Brier score is proper and even though \rocauc{} is generally not proper~\cite{byrne2016note}, we show in Section \ref{sec:roc_auc_proper} that under common assumptions it is. Properness is not defined in respect to \cms{}. Hence, for these metrics, we use the notion of consistency. Roughly speaking, a learning algorithm is consistent in respect to a metric when, given enough training data, the algorithm converges to a solution that optimizes the metric. The consistency of \logloss{} optimization in respect to a large class of \cms{} (including balanced accuracy, $F_{\beta}$, AM, Jaccard similarity coefficient and more) was proved in \cite{menon2013statistical,koyejo2014consistent} where consistency is achieved by optimizing the decision threshold on held-out data. On the other hand, \logloss{} optimization is generally not consistent when using a fixed decision threshold.

\subsection{ROC-AUC}
\label{sec:roc_auc_proper}
In machine learning, it is commonly assumed that the observations comprising the dataset were sampled i.i.d. from an unknown distribution. While AUC is generally not proper~\cite{byrne2016note}, Theorem \ref{thm:auc_proper} below states that under that assumption it is.

\begin{thm}[AUC is proper]
\label{thm:auc_proper}
The model that predicts the true probabilities $Pr[y_i=1 | x_i]$ optimizes the expected \rocauc{} for a set of $N$ observations $(x_i, y_i)_{i=1}^N$ sampled i.i.d. from a distribution.
\end{thm}

\begin{proof}
The process of obtaining a dataset of size $N$ sampled i.i.d. from a distribution could be split into two phases: (1) sample i.i.d. $N$ feature vectors $\{x_i\}_{i=1}^N$ from the distribution and (2) sample each label $y_i$ independently from the corresponding marginal distribution $y_i|x_i$.  Considering that process, Theorem \ref{thm:auc_proper} is a direct corollary of Theorem 7 of \cite{byrne2016note}, restated below as Lemma \ref{lem:byrne}.
\end{proof}

\begin{lem}[Theorem 7 of \cite{byrne2016note}]
\label{lem:byrne}
The model that predicts the true probabilities $Pr[y_i=1 | x_i]$ optimizes the expected \rocauc{} for any set of $N$ feature vectors $(x_i)_{i=1}^N$ where the labels $y_i$ are sampled i.i.d.
\end{lem}

\section{Experimental Setup}
\label{sec:setup}
In this section we describe the experimental setup used in the study.

\subsection{Data}
We considered the $104$ datasets of \cite{kovacs2019empirical} and $24$ dataset of \emph{imbalanced-learn}\footnote{\href{https://imbalanced-learn.org/stable/}{https://imbalanced-learn.org/}} but not in \cite{kovacs2019empirical}. We removed $48$ datasets for which the baseline \rocauc{} score Catboost achieved was very high. For $19$ datasets the baseline score was $1$ and for $29$ datasets the score was higher than $0.99$. We removed these saturated datasets because our goal is to test balancing methods aiming at improving prediction performance and it is unreasonable to expect such improvements for datasets where the baseline score is almost perfect. An additional $7$ datasets were removed due to various technical reasons: The datasets winequality-white-9\_vs\_4, zoo-3 and ecoli-0-1-3-7\_vs\_2-6 had less than $8$ minority samples which caused SMOTE to fail. The datasets habarman was not found in the keel repository and for winequality-red-3\_vs\_5 the baseline score was $0$. SVM-SMOTE failed for winequality-red-8\_vs\_6-7 and poker-9\_vs\_7 so they were also removed. Finally, we experimented with the remaining $73$ datasets.

\subsection{Preprocessing}
String features were encoded as ordinal integers. No preprocessing was applied to the numeric features. The target majority class was encoded with 0 and the minority with 1.

\subsection{Oversamplers}
We experimented with the following common oversamplers\footnote{The implementation of \emph{imbalanced-learn} was used for all oversamplers but Poly, for which the implementation of \emph{smote-variants} was used, \href{https://github.com/analyticalmindsltd/smote\_variants}{https://github.com/analyticalmindsltd/smote\_variants}}:
\begin{description}[leftmargin=0.0cm]
\item[Random Oversampling (Random),] where minority samples are randomly duplicated.
\item[SMOTE~\cite{SMOTE}.]
\item[SVM-SMOTE (SVM-SM)~\cite{SVM-SMOTE},] a SMOTE variant aiming to create synthetic minority samples near the decision line using a kernel machine.
\item[ADASYN~\cite{he2008adasyn},] an adaptive synthetic oversampling approach aiming to create minority samples in areas of the feature space which are harder to predict.
\item[Polynomfit SMOTE (Poly)~\cite{Gazzah},] a variant of SMOTE that was selected due to its superior performance in the experiments of \cite{kovacs2019empirical}.
\end{description}

\subsection{Classifiers}
We experimented with seven classifiers\footnote{We have used the implementation of scikit-learn (1.0.1) for decision tree, SVM, MLP and Adaboost. The python implementation was used for LGBM (3.3.1), XGBoost (1.5.0) and Catboost (1.0.1)}. A simple decision tree was used due to its similarity to the C4.5 learning algorithm used in many other SMOTE studies (\eg, \cite{SMOTE,han2005borderline}). Following the empirical investigations of \cite{kovacs2019empirical,lopez2013insight}, a support vector machine (SVM) and multi layer perceptron (MLP) classifiers were considered. Adaboost\cite{schapire1990strength} was included because numerous boosting methods for imbalanced data are based on it~\cite{galar2011review,LOPEZ20126585}. Finally, three SOTA boosted forest algorithms were used: XGBoost~\cite{XGBoost}, Catboost~\cite{prokhorenkova2018catboost} and lightGBM~\cite{lgbm}. For each classifier, four HP configurations were considered, see Table \ref{tab:classifier_hps}.
\begin{table}
    \begin{tabular}{ll}
        & \textbf{HP Configurations} \\ 
        \toprule
        LGBM & $\eta = 0.1$, $N = 100$, $subsample = 1$ \\
        & $\eta = 0.1$, $N = 100$, $subsample = 0.66$ \\
        & $\eta = 0.025$, $N = 1000$, $subsample = 1$ \\
        & $\eta = 0.025$, $N = 1000$, $subsample = 0.66$ \\
        \midrule
        XGBoost & $\eta = 0.1$, $N = 1000$, $subsample = 1$ \\
        & $\eta = 0.1$, $N = 1000$, $subsample = 0.66$ \\
        & $\eta = 0.025$, $N = 1000$, $subsample = 1$ \\
        & $\eta = 0.025$, $N = 1000$, $subsample = 0.66$ \\
        \midrule
        Catboost & Default (heuristics-based) \\
        & $N = 500$ \\
        & boosting\_type = ordered \\
        & $N = 500$, boosting\_type = ordered \\
        \midrule
        SVM & $C=1$, loss = squared\_hinge \\
        & $C=1$, loss = hinge \\
        & $C=10$, loss = squared\_hinge \\
        & $C=10$, loss = hinge \\
        \midrule
        MLP & hidden\_layer = $10\%$, activation = relu \\
        & hidden\_layer = $50\%$, activation = relu \\
        & hidden\_layer = $100\%$, activation = relu \\
        & hidden\_layer = $50\%$, activation = logistic \\
        \midrule
        Decision & min\_samples\_leaf = $0.5\%$, $1\%$, $2\%$, \\
        tree & $4\%$ of samples \\
        \midrule
        Adaboost & $N = 10,20,40,80$ \\
        \midrule
    \end{tabular}
    \caption{Classifier HP configurations used in our experiments.}
    \label{tab:classifier_hps}
\end{table}

\subsection{Metrics \& Consistency}
We experimented with the following proper \prms{}: \rocauc{}, \logloss{} and Brier score; and with the following \cms{}: $F_1$, $F_2$, Jaccard similarity coefficient and balanced accuracy. For the \cms{}, we experimented with non-consistent classifiers by using a fixed decision threshold of $0.5$ and with consistent ones by optimizing the decision threshold on the validation fold, as proposed in \cite{koyejo2014consistent}.

\subsection{Hyper-parameters (HPs)}
All balancing methods have some HPs that needs to be set. For example, an important one that all balancing techniques share is the desired ratio between positive and negative samples --- in other words, how many minority samples to generate. A common HP selection practice is to use the HP set that yields the best results on the testing data. This practice was used, for example, in the large empirical study of \cite{kovacs2019empirical} and when SMOTE was first proposed \cite{SMOTE}. This may be reasonable when the scientist has previous experience with similar data and knows a-priori how to properly set the HPs. When no such prior knowledge exists, it is common to optimize the HPs on a validation fold. We will denote the former practice \optHPs{} and the latter \valHPs{}. The oversampler HPs used in our experiments are detailed in Table \ref{tab:oversampler_hps}. In addition to the HPs listed in Table \ref{tab:oversampler_hps}, desired ratios of $0.1$, $0.2$, $0.3$, $0.4$ and $0.5$ were considered for all oversamplers.

\begin{table}
    \begin{tabular}{ll}
        & \textbf{HP Configurations} \\ 
        \toprule 
        SMOTE \& ADASYN & $k_n = 3, 5, 7, 9$ \\
        \midrule
        SVM-SMOTE & $k_n = 4$, $m_n = 8$ \\
        & $k_n = 4$, $m_n = 12$ \\
        & $k_n = 6$, $m_n = 8$ \\
        & $k_n = 6$, $m_n = 12$ \\
        \midrule
        Poly & topology = star, bus, mesh \\
        \midrule
    \end{tabular}
    \caption{Oversampler HP configurations used in our experiments.}
    \label{tab:oversampler_hps}
\end{table}

\subsection{Method} Each dataset was randomly stratified split into training, validation and testing folds with ratios 60\%, 20\% and 20\% respectively. To evaluate the oversampling methods, the training fold was oversampled using each of the oversamplers. Each of the classifiers was trained on the augmented training fold with early stopping on the (not-augmented) validation fold when applicable. For the \cms, in order to make the classifiers consistent, the threshold was optimized on the validation fold, as was suggested in \cite{menon2013statistical}. Finally, the metrics were calculated for the validation and test folds. For each set of \{dataset, oversampler, oversampler-HPs, classifier, classifier-HPs\} the experiment was repeated seven times using different random seeds and different data splits.
\section{Experimental Results}
\label{sec:results}
Next we present our results, which were obtained using the experimental setup described in Section \ref{sec:setup}. In our study, we seek to understand the effects of balancing on prediction performance considering several factors, including (i) weak vs strong classifiers; (ii) \optHPs{} vs. \valHPs{}; and (iii) balancing vs optimizing the decision threshold for \cms{}.

\subsection{A-priori vs. validation HPs}
\label{sec:apriori-vs-validation}
First we investigate the effects of HPs on oversampler performance. To avoid the complexity arising from \cms{} and the need to select a decision threshold we use \rocauc{}, the proper (and popular) metric. See Figures \ref{fig:hps_auc}(a) and \ref{fig:hps_auc}(b) for \rocauc{} results with \emph{a-priori} and \valHPs{} respectively. The purple dots represent mean AUC and the orange triangles correspond to mean rank (lower is better), both aggregated over all datasets and random seeds. The error bars correspond to one standard deviation of the mean or rank over the random seed. We follow \cite{lopez2013insight} and also report the rank as an analysis of the mean score could be sensitive to outliers: a single dataset with a very high or low score could have a large impact on the average. The figure is separated vertically into seven parts, one for each classifier. In each part, the first value, marked with "+Default", is the the classifier with default HPs trained without any data augmentations. This was added as a sanity check to verify the utility of our HP optimization process. The second value is the baseline, \ie, a classifier with optimized HPs and no balancing. The remaining five values correspond to oversampling the training data using the various methods.

With \optHPs{}, presented in Figure \ref{fig:hps_auc}(a), all four SMOTE variants achieved considerably better prediction quality than their respective baselines. They were also better than the simple random oversampler. Among the four SMOTE variants, Poly performed slightly better, in accordance with the results of \cite{kovacs2019empirical}. It's important to note that balancing considerably increases the number of HPs. Thus, there is a greater risk of overfitting them when tuning, and therefore, the benefits of oversampling with \optHPs{} observed in our experiments could be attributed to overfitting the HPs. These experiments are included in our study for two reasons: First, we believe that exceptionally good HPs could rise from domain expertise and second, this is a common practice in previous investigations.

Considering \valHPs{}, presented in Figure \ref{fig:hps_auc}(b), balancing significantly improved prediction for the weak classifiers: MLP, SVM, decision tree, Adaboost and LGBM. However, for the stronger classifiers, XGBoost and Catboost, it did not. For XGBoost, no oversampler achieved either mean or rank improvement of more than one standard deviation. The strong Catboost baseline yielded the best average \rocauc{} overall and had the second best rank --- after Catboost oversampled with SVM-SMOTE.

Thus, in the common scenario, when interested in a proper metric and good oversampler HPs are not a-priori known, best prediction is achieved by using a strong classifier and balancing is not beneficial. However, balancing is effective when using a weak classifier or when exceptionally good HPs are a-priori available.

\subsection{The effects of balancing on \rocauc{} and \logloss{}}
\label{sec:logloss_eval}
To gain further insights into balancing, see Table \ref{tab:loglossclasses} which details the \logloss{} and \rocauc{} achieved on the test set with \valHPs. From the \logloss{} scores, it appears that the superior \rocauc{} performance of Catboost was achieved simply by predicting the underlying probabilities better, \ie{} lower \logloss{}. In other words, a superior model for the proper non-symmetric objective (\rocauc) was achieved by optimizing the symmetric loss (\logloss) better. 

The effects of balancing could be observed in the two central columns providing \logloss{} scores for the minority and majority classes separately. In all cases, balancing the data results in better \logloss{} for the minority samples, worse \logloss{} for the majority samples and worse \logloss{} overall. Thus, balancing was successful in shifting the focus of the classifiers toward the minority class. However, while for the weaker classifiers (MLP, SVM, decision tree, Adaboost and LGBM) this shift is correlated with better \rocauc{} performance, for the stronger classifiers, it is not. We leave explaining this disparity between weak and strong classifiers for future work. Note that similar analysis was performed in \cite{weiss2003learning}, however only the decision tree classifier was considered.

\subsection{Balancing vs. threshold optimization}
\label{sec:consistent_eval}
As discussed in Section \ref{sec:proper_consistent}, \logloss{} optimization is not consistent in respect to the popular $F_1$ metric when a fixed decision threshold is used. However, it could be made consistent by optimizing the decision threshold. While optimizing the decision threshold is not a common practice, in what follows we empirically demonstrate that it could have a significant impact on prediction performance, and also on the question of whether or not balancing should be applied.

Similarly to the \rocauc{} experiments, we start by evaluating the benefits of balancing. See Figure \ref{fig:f1_consistent} for average $F_1$ scores and rank using (a) a fixed decision threshold of $0.5$ and (b) an optimized decision threshold, both with \valHPs. When using a fixed threshold, balancing considerably improved prediction performance for all classifiers. Note, however, that the SMOTE-like methods were not significantly better than the simple random oversampler. On the other hand, when optimizing the threshold, balancing was beneficial only for the weak MLP and SVM\footnote{The decision threshold could not be optimized for SVM as it outputs only class predictions - and not class probabilities} classifiers.

\begin{table}
    \centering
    \begin{tabular}{llll}
& & \textbf{Threshold = $0.5$} & \textbf{Optimized threshold}\\ 
\toprule 
\parbox[t]{0pt}{\multirow{5}{*}{\rotatebox[origin=c]{90}{MLP}}} & Baseline & {0.188 $\pm$ 0.044} & {0.303 $\pm$ 0.055}\\ 
& Random & {0.398 $\pm$ 0.077} & \\ 
& SMOTE & \textbf{0.411 $\pm$ 0.080} & \\ 
& SVM-SM & {0.399 $\pm$ 0.078} & \\ 
& ADASYN & {0.404 $\pm$ 0.078} & \\ 
& Poly & {0.411 $\pm$ 0.078} & \\ 
\midrule 
\parbox[t]{0pt}{\multirow{5}{*}{\rotatebox[origin=c]{90}{SVM}}} & Baseline & {0.366 $\pm$ 0.021} & {0.366 $\pm$ 0.021}\\ 
& Random & {0.475 $\pm$ 0.013} & \\ 
& SMOTE & \textbf{0.483 $\pm$ 0.012} & \\ 
& SVM-SM & {0.482 $\pm$ 0.016} & \\ 
& ADASYN & {0.482 $\pm$ 0.012} & \\ 
& Poly & {0.479 $\pm$ 0.012} & \\ 
\midrule 
\parbox[t]{0pt}{\multirow{5}{*}{\rotatebox[origin=c]{90}{DT}}} & Baseline & {0.443 $\pm$ 0.017} & {0.471 $\pm$ 0.020}\\ 
& Random & {0.467 $\pm$ 0.017} & \\ 
& SMOTE & {0.462 $\pm$ 0.023} & \\ 
& SVM-SM & \textbf{0.477 $\pm$ 0.026} & \\ 
& ADASYN & {0.458 $\pm$ 0.016} & \\ 
& Poly & {0.472 $\pm$ 0.010} & \\ 
\midrule 
\parbox[t]{0pt}{\multirow{5}{*}{\rotatebox[origin=c]{90}{Ada}}} & Baseline & {0.454 $\pm$ 0.023} & {0.482 $\pm$ 0.015}\\ 
& Random & {0.488 $\pm$ 0.012} & \\ 
& SMOTE & {0.487 $\pm$ 0.016} & \\ 
& SVM-SM & {0.497 $\pm$ 0.018} & \\ 
& ADASYN & {0.491 $\pm$ 0.018} & \\ 
& Poly & \textbf{0.497 $\pm$ 0.019} & \\ 
\midrule 
\parbox[t]{0pt}{\multirow{5}{*}{\rotatebox[origin=c]{90}{LGBM}}} & Baseline & {0.398 $\pm$ 0.013} & {0.504 $\pm$ 0.015}\\ 
& Random & {0.509 $\pm$ 0.015} & \\ 
& SMOTE & {0.513 $\pm$ 0.019} & \\ 
& SVM-SM & {0.500 $\pm$ 0.016} & \\ 
& ADASYN & {0.510 $\pm$ 0.022} & \\ 
& Poly & \textbf{0.514 $\pm$ 0.014} & \\ 
\midrule 
\parbox[t]{0pt}{\multirow{5}{*}{\rotatebox[origin=c]{90}{XGB}}} & Baseline & {0.436 $\pm$ 0.013} & \textbf{0.515 $\pm$ 0.018}\\ 
& Random & {0.508 $\pm$ 0.022} & \\ 
& SMOTE & {0.513 $\pm$ 0.023} & \\ 
& SVM-SM & {0.507 $\pm$ 0.023} & \\ 
& ADASYN & {0.504 $\pm$ 0.019} & \\ 
& Poly & {0.509 $\pm$ 0.012} & \\ 
\midrule 
\parbox[t]{0pt}{\multirow{5}{*}{\rotatebox[origin=c]{90}{Cat}}} & Baseline & {0.444 $\pm$ 0.016} & \textbf{0.532 $\pm$ 0.020}\\ 
& Random & {0.525 $\pm$ 0.019} & \\ 
& SMOTE & {0.529 $\pm$ 0.022} & \\ 
& SVM-SM & {0.520 $\pm$ 0.024} & \\ 
& ADASYN & {0.523 $\pm$ 0.015} & \\ 
& Poly & {0.524 $\pm$ 0.012} & \\ 
\midrule 
\end{tabular}
    \caption{Mean $F_1$ score achieved by balancing the data and by optimizing the decision threshold. As detailed above and discussed in Section \ref{sec:consistent_eval}, except for very weak classifiers (MLP and SVM), balancing the data and optimizing the decision threshold yield comparable prediction performance.\vspace{-20pt}}
    \label{tab:f1_def_opt}
\end{table}

To compare balancing to optimizing the decision threshold, we compare the $F_1$ scores achieved by balancing the data and optimizing the threshold; see Table \ref{tab:f1_def_opt}.
\begin{description}[leftmargin=*]
\item[MLP \& SVM] For the weak MLP and SVM, best prediction was achieved by oversampling with SMOTE. Furthermore, prediction quality was considerably better compared to optimizing the decision threshold.
\item[Decision tree, Adaboost \& LGBM] Best $F_1$ was achieved by SVM-SMOTE or Poly. However, the score was only slightly better than the score achieved by optimizing the threshold.
\item[XGBoost \& Catboost] For the strongest classifiers, best prediction was achieved by optimizing the threshold (and without balancing). Nevertheless, the score was not significantly better than the score achieved by balancing.
\end{description}
To summarize, except for very weak classifiers (MLP and SVM), balancing the data and optimizing the decision threshold yield comparable prediction performance.

\subsection{Additional experiments}
\label{sec:additional_experiments}
To validate the robustness of our results and insensitivity to specific experimental design choices, we ran additional experiments while slightly modifying our setup. These include applying normalization to the data before oversampling it, using two validation folds (one for early stopping and one for HP selection) and using $5$-fold cross validation. All additional experiments yield results similar to the ones reported here. These additional experiments, along with full results for all metrics (\rocauc{}, \logloss{}, Brier score, $F_1$, $F_2$, Jaccard similarity coefficient and balanced accuracy), can be found in our code repository\footnote{\ifbool{anonymous}
{\href{https://github.com/anonymous}{https://github.com/anonymous}}
{\href{https://github.com/aws/to-smote-or-not}{https://github.com/aws/to-smote-or-not}}
}.
\section{Discussion and Recommendations}
\label{sec:discussion}
We first discuss how best prediction performance can generally be achieved and then describe some specific scenarios which require different methods. Usually, optimal oversampler HPs are not a-priori known and have to be learned from the data. To avoid text repetition, we assume the HPs not a-priori known unless stated otherwise.

\subsection{General recommendation}
Normally, balancing is not recommended. However, the reasons differ for proper metrics and \cms{}:

\smallskip
\noindent
\textbf{Proper metric.} When the objective metric is proper, best prediction is achieved by using a strong classifier and balancing the data is not required nor beneficial.

\smallskip
\noindent
\textbf{Label metric.} When using a \cm{} the decision threshold has to be somehow set. A common method is to use a fixed threshold of $0.5$. A known but less common technique is to optimize the decision threshold on the validation data after model training. Our experiments show that optimizing the decision threshold yield similar prediction quality to balancing the data. However, optimizing the decision threshold is recommended because it's more compute efficient: When optimizing the decision threshold, the model is trained once and probability predictions are inferred for the validation set. Then, optimizing the decision threshold only involves recalculating the objective metric using several thresholds as required by the optimization method. On the other hand, optimizing the oversampler HPs requires refitting the classifier several times which is considerably more compute expensive. Moreover, optimizing the threshold allows for a single trained model to be tuned to optimize multiple objective metrics.

\subsection{Scenarios where balancing is recommended}
\label{sec:balancing_good}
Although balancing is generally not recommended, there are a few scenarios in which oversampling should be applied:

\smallskip
\noindent
\textbf{Using a \cm{} \& threshold optimization is impossible.} In some cases, it's not possible to optimize the decision threshold, \eg, when using legacy classifiers or due to other system constraints. We have demonstrated empirically, that when using a fixed threshold, balancing is beneficial. However, for strong classifiers, the sophisticated SMOTE variants were not more effective than the simple random oversampler. Thus, when using a strong but non-consistent classifier, it is recommended to augment the data using the simple random oversampler.

\smallskip
\noindent
\textbf{Weak classifiers.} We have seen that SMOTE-like methods could improve the performance of weak classifiers such as MLP, SVM, decision tree, Adaboost and LGBM. Thus, when strong classifiers can't be used, it is recommended to augment the data using the SMOTE-like techniques. Nevertheless, the resulting prediction quality will be significantly worse compared to using a strong classifier (without oversampling).

\smallskip
\noindent
\textbf{\OptHPs.} When exceptionally good HPs are known, oversampling with the SMOTE variants improves prediction performance for all classifiers. Note that this could not be achieved by tuning the HPs on a validation set. Rather, we refer to HPs that are known \emph{a-priori} to yield good results, for example, from previous experience with similar data. This is not common and usually rises from domain expertise. Furthermore, as balancing greatly increases the number of HPs to be tuned, it's unclear how much of the gains demonstrated using the HPs that yield best results on the test data could be attributed to overfitting them.

\subsection{Other difficulty factors}
As mentioned in Section \ref{sec:intro}, some studies propose that skewed class distribution is not the root of challenge in imbalanced classification but rather that the difficulty arises from other phenomenons that are common in this setting~\cite{lopez2013insight,stefanowski2016dealing,sun2009classification}. We would like to emphasize that our experimental results were obtained using a large collection of real-life imbalanced datasets which probably exhibit many of the other contributing difficulty factors. Thus, our findings hold for real imbalanced data whether the source of difficulty is the skewed class distribution itself or other phenomenons.

\section{Related Work}
\label{sec:related_work}
Our results seem to contradict previous studies demonstrating the benefits of balancing. That is because most previous empirical work focused on the scenarios where balancing is beneficial, outlined in Section \ref{sec:balancing_good}. A closer look reveals that our results are, in fact, in accord with most previous study when carefully considering the scenarios under investigation. In particular, most previous works do not to include SOTA classifiers as baselines (some of these studies precede the invention of the strong boosting methods available to us today). For example, the initial work on SMOTE \cite{SMOTE} considered the weak C4.5 (decision tree) and Ripper algorithms; \cite{SVM-SMOTE} and \cite{Gazzah} experimented with support vector machine classifiers only; decision trees were considered in \cite{he2008adasyn}. Few large scale empirical studies of balancing schemes were conducted: \cite{kovacs2019empirical} compared $85$ balancing methods using $104$ datasets. In that study, balancing was beneficial due to the use of weak classifiers and \optHPs{}. The study of \cite{LOPEZ20126585} compared the \rocauc{} performance of several balancing methods and several classifiers dedicated to imbalanced data and found several SMOTE-like methods to be beneficial for SVM, decision tree and KNN classifiers. Several SMOTE variants were studied in \cite{batista2004study} and were shown to improve performance for C4.5. Bootstrapping techniques improved prediction quality of C4.5 in \cite{weiss2003learning}.

Studies of balancing that include SOTA classifiers are rare. For example, \cite{al2020predicting} demonstrated how SMOTE can improve the $F_1$ score of XGBoost when using a fixed decision threshold. Several specialized boosting schemes for imbalanced data were proposed, see the surveys of \cite{galar2011review,sun2009classification,tanha2020boosting}. However, these are based on and compared to the weak Adaboost.

As far as we could understand from the texts, in most studies of balancing that considered \cms{}, a fixed decision threshold of $0.5$ was used~\cite{kovacs2019empirical,gosain2017handling}. Other studies considered optimizing the decision threshold or similar methods but they were not compared to balancing~\cite{hernandez2012unified,thai2010learning}. Note that optimizing the decision threshold in order to account for imbalance is similar to fitting the bias in SVMs~\cite{nunez2011post}.

Several papers have questioned the utility of balancing. Recently, it was demonstrated in \cite{StopOversampling9761871} that some of the synthetic minority samples generated by several SMOTE variants should truly belong to the majority class --- potentially hindering the performance of a classifier trained on such data. The experiments of \cite{maloof2003learning} showed that optimizing the decision threshold is as effective as random oversampling. However, this study did not consider SMOTE but rather only the basic random oversampling and undersampling. Moreover, the study considered a single dataset only.
\section{Conclusion}
\label{sec:conclusion}

In this work, we evaluated the utility of balancing for imbalanced binary classification tasks by performing extensive experiments and including strong SOTA classifiers as baselines. We have empirically demonstrated that balancing could be beneficial for weak classifiers but not for strong ones. Furthermore, best prediction performance was achieved using the strong Catboost classifier and without balancing. While balancing is generally not recommended, several scenarios in which it should be applied were discussed. This clarifies the gap between our results and the benefits of oversampling demonstrated in previous studies that focus on these settings (where balancing is beneficial).

We have empirically demonstrated that balancing is effective in shifting the focus of the classifiers towards the minority class resulting in better prediction of the minority samples and worse prediction of the majority samples. This correlated with better prediction performance for the weak classifiers but not for the strong ones. We believe that understanding this disparity is an interesting avenue for future work. Another interesting research direction is evaluating the utility of specialized imbalanced classifier schemes and post-processing techniques using SOTA classifiers.

\begin{acks}
We would like to thank \ifbool{anonymous}{}{Saharon Rosset and Zohar Karnin for fruitful discussion and }the authors of \cite{keel,UCI_datasets,coil_dataset,crime_dataset,mammograph_dataset,promise} for allowing the use of their data.
\end{acks}

\begin{figure*}[p]
    \centering
    \begin{subfigure}{.5\textwidth}
        \centering
        \includegraphics[width=0.95\textwidth]{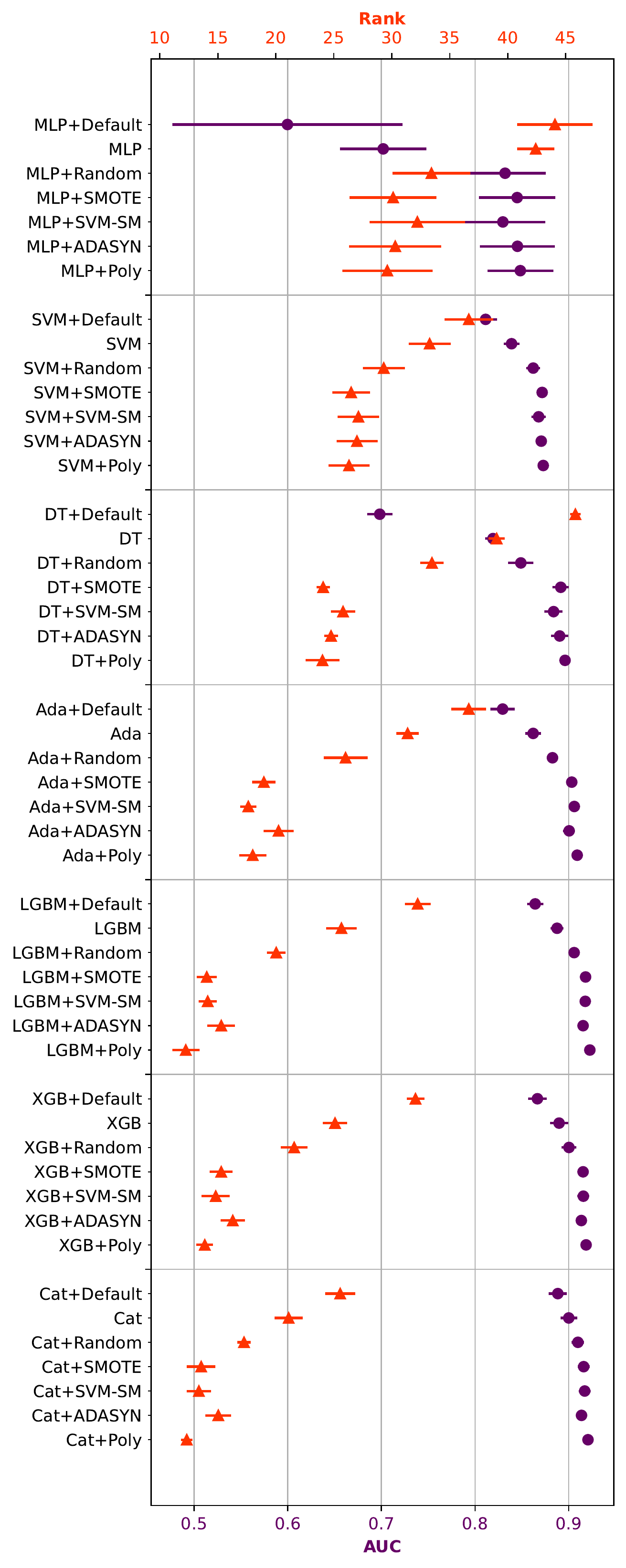}
        \vspace{-5pt}
        \caption{\OptHPs}
    \end{subfigure}%
    \begin{subfigure}{.5\textwidth}
        \centering
        \includegraphics[width=0.95\textwidth]{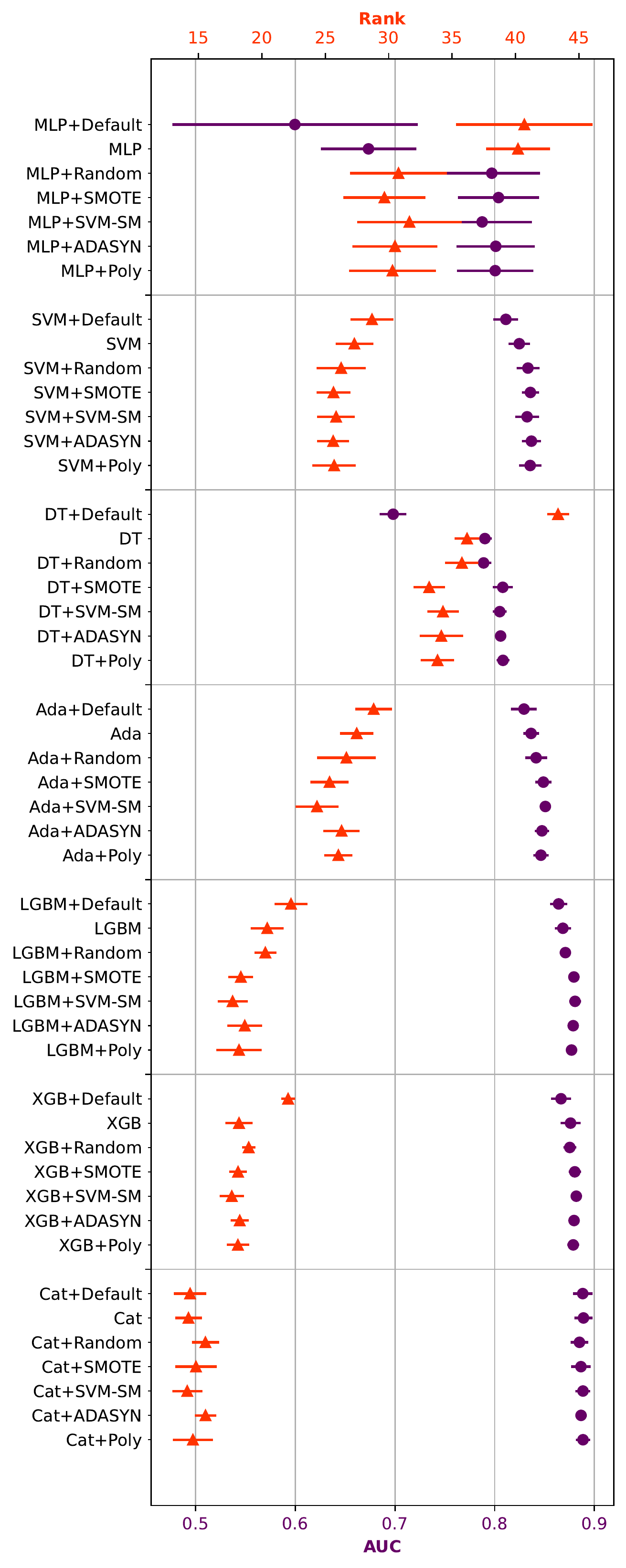}
        \vspace{-5pt}
        \caption{\ValHPs}
    \end{subfigure}
    \vspace{-5pt}
    \caption{Mean \rocauc{} and rank (a) using \optHPs{} and (b) using \valHPs. The purple dots represent \rocauc{} averaged over all datasets and the orange triangles corresponds to the rank (lower is better). As illustrated above and further detailed in Section \ref{sec:apriori-vs-validation}, oversampling the data does not yield significant improvements for the strong classifiers when considering \valHPs{}.}
    \label{fig:hps_auc}
\end{figure*}

\begin{table*}[p]
    \centering
    \begin{tabular}{lllllll}
& & \textbf{Logistic loss ($\cdot 10^{3}$)} & \textbf{Logistic loss ($\cdot 10^{3}$)} & \textbf{Logistic loss ($\cdot 10^{3}$)} & \textbf{AUC ($\cdot 10^{3}$)} & \textbf{AUC} \\ 
& & & \textbf{(minority samples)} & \textbf{(majority samples)} & \textbf{(mean)} & \textbf{(rank)}\\ 
\toprule 
\parbox[t]{0pt}{\multirow{5}{*}{\rotatebox[origin=c]{90}{MLP}}} & Default & 1084 $\pm$ 548 (-38\%) & 6370 $\pm$ 789 (+87\%) & 522 $\pm$ 608 (-67\%) & 600 $\pm$ 123 (-11\%) & 40.7 $\pm$ 5.4\\ 
& Baseline & 1760 $\pm$ 1397 & 3402 $\pm$ 771 & 1588 $\pm$ 1495 & 673 $\pm$ 48 & 40.2 $\pm$ 2.5\\ 
& Random & 1639 $\pm$ 1186 (-6.9\%) & 1770 $\pm$ 660 (-48\%) & 1600 $\pm$ 1238 (+0.8\%) & 797 $\pm$ 48 (+18\%) & 30.8 $\pm$ 3.8\\ 
& SMOTE & 1494 $\pm$ 1279 (-15\%) & 1774 $\pm$ 626 (-48\%) & 1477 $\pm$ 1353 (-7.0\%) & 804 $\pm$ 41 (+19\%) & 29.7 $\pm$ 3.2\\ 
& SVM-SM & 1556 $\pm$ 1260 (-12\%) & 2054 $\pm$ 845 (-40\%) & 1499 $\pm$ 1302 (-5.6\%) & 787 $\pm$ 50 (+17\%) & 31.6 $\pm$ 4.1\\ 
& ADASYN & 1603 $\pm$ 1213 (-8.9\%) & 1751 $\pm$ 732 (-49\%) & 1586 $\pm$ 1294 (-0.1\%) & 801 $\pm$ 39 (+19\%) & 30.5 $\pm$ 3.4\\ 
& Poly & 1533 $\pm$ 1191 (-13\%) & 1755 $\pm$ 610 (-48\%) & 1498 $\pm$ 1224 (-5.7\%) & 801 $\pm$ 38 (+19\%) & 30.3 $\pm$ 3.4\\ 
\midrule 
\parbox[t]{0pt}{\multirow{5}{*}{\rotatebox[origin=c]{90}{SVM}}} & Default & 4069 $\pm$ 656 (+8.0\%) & 24207 $\pm$ 703 (+4.1\%) & 2408 $\pm$ 748 (+19\%) & 811 $\pm$ 12 (-1.6\%) & 28.7 $\pm$ 1.7\\ 
& Baseline & 3768 $\pm$ 553 & 23243 $\pm$ 813 & 2025 $\pm$ 639 & 825 $\pm$ 11 & 27.3 $\pm$ 1.5\\ 
& Random & 4988 $\pm$ 380 (+32\%) & 15275 $\pm$ 702 (-34\%) & 4159 $\pm$ 586 (+105\%) & 833 $\pm$ 12 (+1.1\%) & 26.3 $\pm$ 1.9\\ 
& SMOTE & 5036 $\pm$ 656 (+34\%) & 14744 $\pm$ 948 (-37\%) & 4278 $\pm$ 801 (+111\%) & 836 $\pm$ 9 (+1.4\%) & 25.6 $\pm$ 1.3\\ 
& SVM-SM & 4172 $\pm$ 443 (+11\%) & 16654 $\pm$ 901 (-28\%) & 3203 $\pm$ 526 (+58\%) & 833 $\pm$ 12 (+0.9\%) & 25.9 $\pm$ 1.5\\ 
& ADASYN & 5388 $\pm$ 477 (+43\%) & 14722 $\pm$ 549 (-37\%) & 4711 $\pm$ 510 (+133\%) & 837 $\pm$ 9 (+1.5\%) & 25.6 $\pm$ 1.3\\ 
& Poly & 5477 $\pm$ 351 (+45\%) & 14817 $\pm$ 594 (-36\%) & 4821 $\pm$ 385 (+138\%) & 836 $\pm$ 11 (+1.3\%) & 25.7 $\pm$ 1.7\\ 
\midrule 
\parbox[t]{0pt}{\multirow{5}{*}{\rotatebox[origin=c]{90}{DT}}} & Default & 3219 $\pm$ 99 (+180\%) & 18686 $\pm$ 943 (+84\%) & 1993 $\pm$ 66 (+310\%) & 698 $\pm$ 13 (-12\%) & 43.4 $\pm$ 0.9\\ 
& Baseline & 1148 $\pm$ 80 & 10154 $\pm$ 554 & 486 $\pm$ 94 & 790 $\pm$ 7 & 36.2 $\pm$ 1.0\\ 
& Random & 1327 $\pm$ 112 (+16\%) & 9879 $\pm$ 572 (-2.7\%) & 704 $\pm$ 95 (+45\%) & 789 $\pm$ 8 (-0.2\%) & 35.8 $\pm$ 1.3\\ 
& SMOTE & 1339 $\pm$ 69 (+17\%) & 7834 $\pm$ 474 (-23\%) & 843 $\pm$ 27 (+74\%) & 808 $\pm$ 10 (+2.3\%) & 33.2 $\pm$ 1.2\\ 
& SVM-SM & 1228 $\pm$ 65 (+7.0\%) & 8329 $\pm$ 565 (-18\%) & 718 $\pm$ 70 (+48\%) & 805 $\pm$ 7 (+1.9\%) & 34.3 $\pm$ 1.2\\ 
& ADASYN & 1363 $\pm$ 102 (+19\%) & 7678 $\pm$ 280 (-24\%) & 874 $\pm$ 99 (+80\%) & 806 $\pm$ 4 (+2.0\%) & 34.2 $\pm$ 1.7\\ 
& Poly & 1225 $\pm$ 140 (+6.6\%) & 7263 $\pm$ 449 (-28\%) & 805 $\pm$ 164 (+66\%) & 808 $\pm$ 6 (+2.3\%) & 33.9 $\pm$ 1.3\\ 
\midrule 
\parbox[t]{0pt}{\multirow{5}{*}{\rotatebox[origin=c]{90}{Ada}}} & Default & 462 $\pm$ 8 (+7.7\%) & 969 $\pm$ 128 (-8.1\%) & 420 $\pm$ 7 (+11\%) & 829 $\pm$ 13 (-0.8\%) & 28.8 $\pm$ 1.4\\ 
& Baseline & 429 $\pm$ 7 & 1055 $\pm$ 160 & 379 $\pm$ 10 & 837 $\pm$ 8 & 27.5 $\pm$ 1.3\\ 
& Random & 445 $\pm$ 10 (+3.7\%) & 970 $\pm$ 151 (-8.0\%) & 401 $\pm$ 9 (+5.9\%) & 842 $\pm$ 11 (+0.6\%) & 26.7 $\pm$ 2.3\\ 
& SMOTE & 460 $\pm$ 13 (+7.1\%) & 919 $\pm$ 145 (-13\%) & 419 $\pm$ 4 (+11\%) & 849 $\pm$ 8 (+1.5\%) & 25.3 $\pm$ 1.5\\ 
& SVM-SM & 450 $\pm$ 10 (+4.8\%) & 882 $\pm$ 96 (-16\%) & 414 $\pm$ 6 (+9.3\%) & 851 $\pm$ 4 (+1.7\%) & 24.3 $\pm$ 1.7\\ 
& ADASYN & 469 $\pm$ 7 (+9.3\%) & 930 $\pm$ 157 (-12\%) & 429 $\pm$ 11 (+13\%) & 847 $\pm$ 7 (+1.3\%) & 26.3 $\pm$ 1.4\\ 
& Poly & 484 $\pm$ 9 (+13\%) & 777 $\pm$ 92 (-26\%) & 459 $\pm$ 13 (+21\%) & 846 $\pm$ 8 (+1.2\%) & 26.0 $\pm$ 1.1\\ 
\midrule 
\parbox[t]{0pt}{\multirow{5}{*}{\rotatebox[origin=c]{90}{LGBM}}} & Default & 185 $\pm$ 2 (+1.2\%) & 1789 $\pm$ 37 (+0.8\%) & 74 $\pm$ 3 (+0.9\%) & 864 $\pm$ 9 (-0.5\%) & 22.3 $\pm$ 1.3\\ 
& Baseline & 182 $\pm$ 2 & 1775 $\pm$ 37 & 73 $\pm$ 3 & 868 $\pm$ 8 & 20.4 $\pm$ 1.3\\ 
& Random & 193 $\pm$ 5 (+6.0\%) & 1488 $\pm$ 62 (-16\%) & 108 $\pm$ 4 (+47\%) & 871 $\pm$ 6 (+0.3\%) & 20.3 $\pm$ 0.9\\ 
& SMOTE & 192 $\pm$ 6 (+5.1\%) & 1437 $\pm$ 51 (-19\%) & 109 $\pm$ 5 (+48\%) & 879 $\pm$ 4 (+1.3\%) & 18.3 $\pm$ 1.0\\ 
& SVM-SM & 188 $\pm$ 4 (+3.0\%) & 1477 $\pm$ 55 (-17\%) & 104 $\pm$ 4 (+41\%) & 881 $\pm$ 6 (+1.4\%) & 17.7 $\pm$ 1.2\\ 
& ADASYN & 197 $\pm$ 6 (+7.8\%) & 1400 $\pm$ 55 (-21\%) & 118 $\pm$ 5 (+62\%) & 879 $\pm$ 4 (+1.2\%) & 18.7 $\pm$ 1.4\\ 
& Poly & 194 $\pm$ 7 (+6.6\%) & 1425 $\pm$ 56 (-20\%) & 116 $\pm$ 9 (+58\%) & 877 $\pm$ 6 (+1.0\%) & 18.2 $\pm$ 1.8\\ 
\midrule 
\parbox[t]{0pt}{\multirow{5}{*}{\rotatebox[origin=c]{90}{XGB}}} & Default & 187 $\pm$ 1 (+4.6\%) & 1620 $\pm$ 43 (+0.6\%) & 90 $\pm$ 4 (+9.7\%) & 867 $\pm$ 10 (-1.1\%) & 22.1 $\pm$ 0.5\\ 
& Baseline & 179 $\pm$ 3 & 1610 $\pm$ 43 & 82 $\pm$ 4 & 876 $\pm$ 10 & 18.2 $\pm$ 1.1\\ 
& Random & 193 $\pm$ 4 (+7.6\%) & 1417 $\pm$ 58 (-12\%) & 113 $\pm$ 6 (+38\%) & 875 $\pm$ 6 (-0.1\%) & 19.0 $\pm$ 0.5\\ 
& SMOTE & 194 $\pm$ 5 (+8.2\%) & 1368 $\pm$ 55 (-15\%) & 117 $\pm$ 2 (+42\%) & 880 $\pm$ 6 (+0.5\%) & 18.1 $\pm$ 0.7\\ 
& SVM-SM & 189 $\pm$ 4 (+5.4\%) & 1394 $\pm$ 52 (-13\%) & 112 $\pm$ 5 (+36\%) & 882 $\pm$ 6 (+0.7\%) & 17.6 $\pm$ 1.0\\ 
& ADASYN & 195 $\pm$ 5 (+9.0\%) & 1351 $\pm$ 46 (-16\%) & 121 $\pm$ 4 (+47\%) & 880 $\pm$ 5 (+0.4\%) & 18.3 $\pm$ 0.7\\ 
& Poly & 193 $\pm$ 6 (+7.7\%) & 1356 $\pm$ 54 (-16\%) & 118 $\pm$ 7 (+44\%) & 879 $\pm$ 6 (+0.3\%) & 18.1 $\pm$ 0.9\\ 
\midrule 
\parbox[t]{0pt}{\multirow{5}{*}{\rotatebox[origin=c]{90}{Cat}}} & Default & 170 $\pm$ 1 (+0.1\%) & 1597 $\pm$ 32 (-0.4\%) & 74 $\pm$ 3 (+0.6\%) & 888 $\pm$ 10 (-0.1\%) & 14.3 $\pm$ 1.3\\ 
& Baseline & 170 $\pm$ 1 & 1603 $\pm$ 31 & 73 $\pm$ 3 & 889 $\pm$ 9 & 14.2 $\pm$ 1.1\\ 
& Random & 179 $\pm$ 3 (+5.4\%) & 1339 $\pm$ 52 (-16\%) & 104 $\pm$ 4 (+41\%) & 885 $\pm$ 9 (-0.5\%) & 15.6 $\pm$ 1.1\\ 
& SMOTE & 180 $\pm$ 5 (+5.9\%) & 1299 $\pm$ 63 (-19\%) & 106 $\pm$ 2 (+44\%) & 887 $\pm$ 10 (-0.3\%) & 14.8 $\pm$ 1.6\\ 
& SVM-SM & 177 $\pm$ 3 (+4.3\%) & 1356 $\pm$ 62 (-15\%) & 101 $\pm$ 4 (+38\%) & 889 $\pm$ 7 (-0.1\%) & 14.1 $\pm$ 1.2\\ 
& ADASYN & 183 $\pm$ 3 (+7.8\%) & 1269 $\pm$ 40 (-21\%) & 113 $\pm$ 2 (+54\%) & 887 $\pm$ 6 (-0.3\%) & 15.6 $\pm$ 0.8\\ 
& Poly & 181 $\pm$ 5 (+6.5\%) & 1292 $\pm$ 34 (-19\%) & 107 $\pm$ 5 (+46\%) & 889 $\pm$ 7 (-0.1\%) & 14.6 $\pm$ 1.6\\ 
\midrule 
\end{tabular}
    \caption{Mean \logloss{} (lower is better) and \rocauc{} with \valHPs. We report (from left to right): the \logloss{}; the \logloss{} of the minority and majority classes; and average \rocauc{} and rank. The numbers reported in brackets are the difference from the baseline (in percentage). The \logloss{} and \rocauc{} scores are multiplied by $10^3$.}
    \label{tab:loglossclasses}
\end{table*}
\begin{figure*}[p]
    \centering
    \begin{subfigure}{.5\textwidth}
        \centering
        \includegraphics[width=0.95\textwidth]{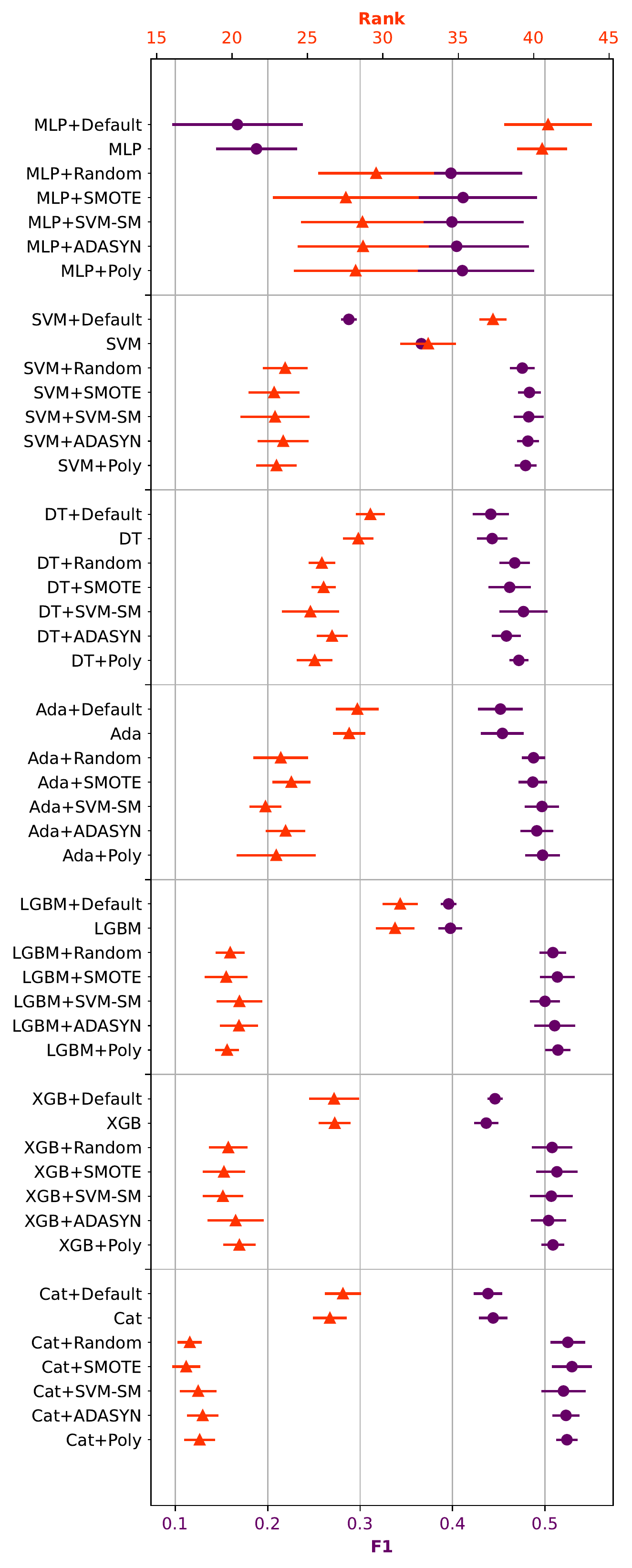}
        \vspace{-5pt}
        \caption{Decision threshold = $0.5$}
    \end{subfigure}%
    \begin{subfigure}{.5\textwidth}
        \centering
        \includegraphics[width=0.95\textwidth]{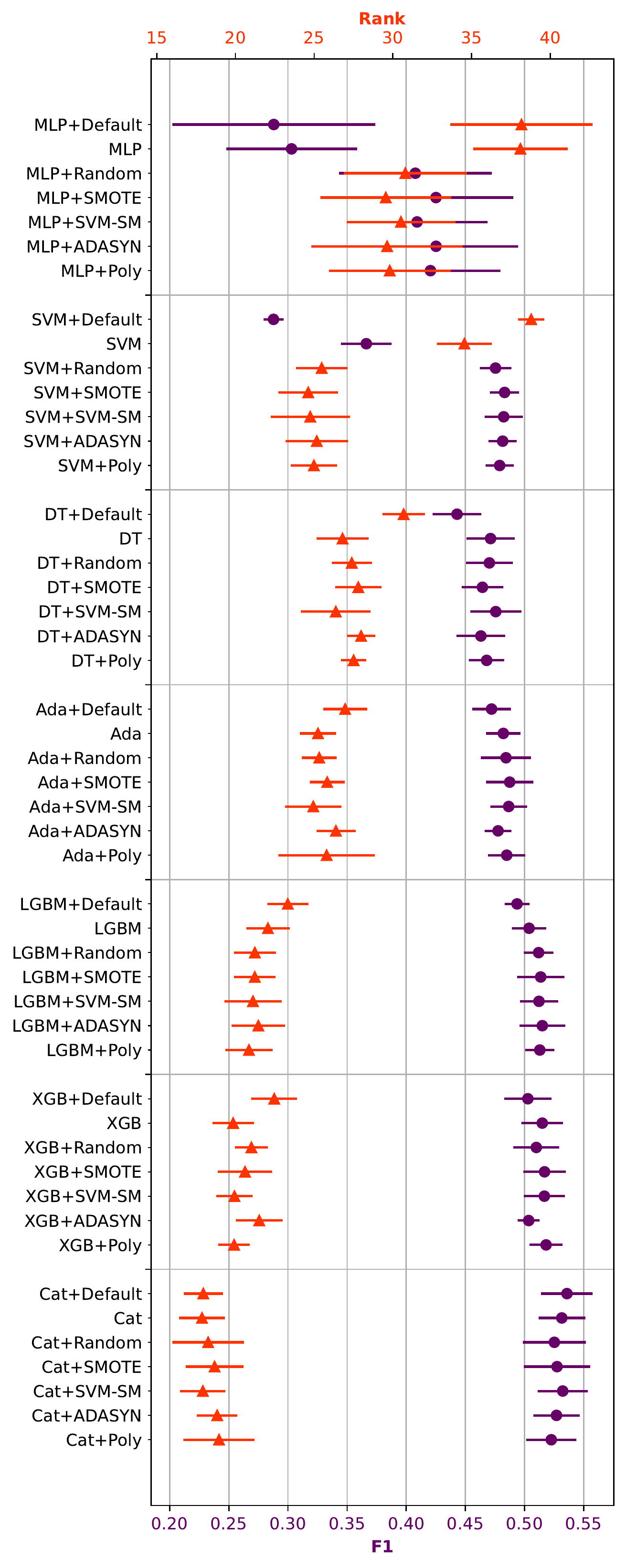}
        \vspace{-5pt}
        \caption{Optimized decision threshold}
    \end{subfigure}
    \vspace{-5pt}
    \caption{Mean $F_1$ score and rank with \optHPs{} (a) using a fixed decision threshold of $0.5$ and (b) optimizing the decision threshold on the validation set. As demonstrated above, while oversampling the data is beneficial when using a fixed threshold, it does not improve prediction quality when the threshold is optimized, see further discussion in Section \ref{sec:consistent_eval}.}
    \label{fig:f1_consistent}
\end{figure*}

\clearpage
\bibliographystyle{ACM-Reference-Format}
\bibliography{references}

\end{document}